\documentclass{article} % For LaTeX2e
\usepackage{iclr2027_conference,times}

\usepackage{booktabs}
\usepackage{amsfonts}
\usepackage{nicefrac}
\usepackage{microtype}
\usepackage{xcolor}
\usepackage{graphicx}
\usepackage{textcomp}
\usepackage{enumitem}
\usepackage{multirow}
\usepackage{pifont}
\usepackage{bm}
\usepackage{amsmath}
\usepackage{amssymb}
\usepackage{amsthm}
\usepackage{bbm}
\usepackage{caption}
\usepackage{float}
\usepackage{subcaption}
\usepackage{algorithm}
\usepackage{algorithmic}
\usepackage{wrapfig}
\usepackage{colortbl}
\usepackage{soul}
\usepackage[colorlinks,urlcolor=blue,linkcolor=blue,citecolor=blue]{hyperref}
\usepackage{url}
\usepackage[capitalize,noabbrev,nameinlink]{cleveref}

\crefname{lem}{Lemma}{Lemmas}
\crefname{defin}{Definition}{Definitions}
\crefname{thm}{Theorem}{Theorems}

\newtheorem{theorem}{\bf Theorem}
\newtheorem{lemma}{\bf Lemma}

\newtheorem{definition}{\bf Definition}

\definecolor{lightgray}{gray}{0.9}
\newif\ifchainlorared
\chainloraredfalse

\title{ChainLoRA: Geometry-Preserving Task Vector Merging for Continual Learning in LLMs}

\author{%
\parbox[t]{\dimexpr\textwidth-2\tabcolsep\relax}{%
\raggedright
{\small\mbox{\textbf{Hang Yin\textsuperscript{1,2}, Haozhe Wang\textsuperscript{1,3,4}, Yuhua Luo\textsuperscript{1}, Zhangqi Pan\textsuperscript{5}, Xiaoxing Wang\textsuperscript{6}, Junchi Yan\textsuperscript{1,2,\textdagger}}}}\\[3pt]
{\normalfont\small
\textsuperscript{1}Shanghai Jiao Tong University, \textsuperscript{2}Shanghai Innovation Institute,\\
\textsuperscript{3}Beijing Zhongguancun Academy, \textsuperscript{4}StepFun,\\
\textsuperscript{5}Huazhong University of Science and Technology, \textsuperscript{6}IAAR\\[2pt]
\textsuperscript{\textdagger}Corresponding author}%
}}

\iclrfinalcopy
\hypersetup{pdfauthor={Hang Yin, Haozhe Wang, Yuhua Luo, Zhangqi Pan, Xiaoxing Wang, Junchi Yan},pdftitle={ChainLoRA: Geometry-Preserving Task Vector Merging for Continual Learning in LLMs}}
\begin{document}

\maketitle
\fancyhead{}
\raggedbottom
\begin{abstract}
Continual parameter-efficient fine-tuning for large language models (LLMs) must balance retention of previously acquired knowledge, adaptation to new tasks, and strict parameter budgets.
We present \textbf{ChainLoRA}, a replay-free continual merging
framework built on chain-updated task-vector geometry.
From a parameter-merging perspective, we formulate a geometric
view of forgetting through a measurable interaction between task
updates, separating directional overlap from coefficient coupling.
Building on this view, ChainLoRA combines chain-updated training
with post-stream adaptive SVD merging. During training,
initialization and a one-sided orthogonality proxy use only the
last carrier, keeping their historical-state footprint and
regularization overhead constant as the task stream grows.
At merging time, Adaptive SVD extracts a shared carrier and
aligns it to the latest task through Procrustes adaptation.
Our theoretical analysis shows that Procrustes adaptation facilitates geometric approximate separation of shared and task-specific components. The one-sided proxy further bounds inter-task interference. An effective-rank penalty additionally promotes efficient utilization of the task subspace during continual learning. Experiments show that ChainLoRA achieves state-of-the-art performance among the evaluated replay-free methods on the Large and SuperNI benchmarks, while remaining competitive on Standard CL and attaining almost the closest average scores to the evaluated replay-based method across all three benchmarks.
\end{abstract}

\section{Introduction}
\begin{wrapfigure}{r}{0.43\textwidth} % r右侧，可改l
    \centering
    \vspace{-10pt}
    \includegraphics[width=\linewidth]{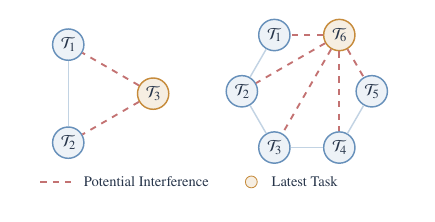}
    \caption{Potential task interference in short CL (left) and long CL (right).}
    \label{fig:inter_task}
\end{wrapfigure}
Foundation models, e.g. large language models (LLMs), have shown remarkable generalization and reasoning abilities across diverse domains. When they are subsequently finetuned on a stream of domains ~\citep{karanikolas2023large}, later updates often overwrite earlier ones---\emph{catastrophic forgetting}~\citep{mccloskey1989catastrophic,liu2024more}. Continual learning~\citep{shi2025continual} studies how to acquire new knowledge without this collapse. \emph{Regularization methods}~\citep{kirkpatrick2017overcoming} keep parameters close to a previous solution, but can hinder adaptation to new tasks. \emph{Replay methods}~\citep{rebuffi2017icarl} revisit old data, which is often unavailable or costly during post-training. \emph{Architectural isolation}~\citep{tong2025any} assigns tasks to separate modules, at the price of extra parameters and weaker sharing~\citep{verma2021efficient}, which may further limit generalization across tasks. These limits raise a concrete question: \emph{Can a CL method avoid catastrophic forgetting without compromising cross-task generalization or relying on data replay?}
\begin{figure*}[tb]
	\centering
	\includegraphics[width=\textwidth,height=0.27\textheight]{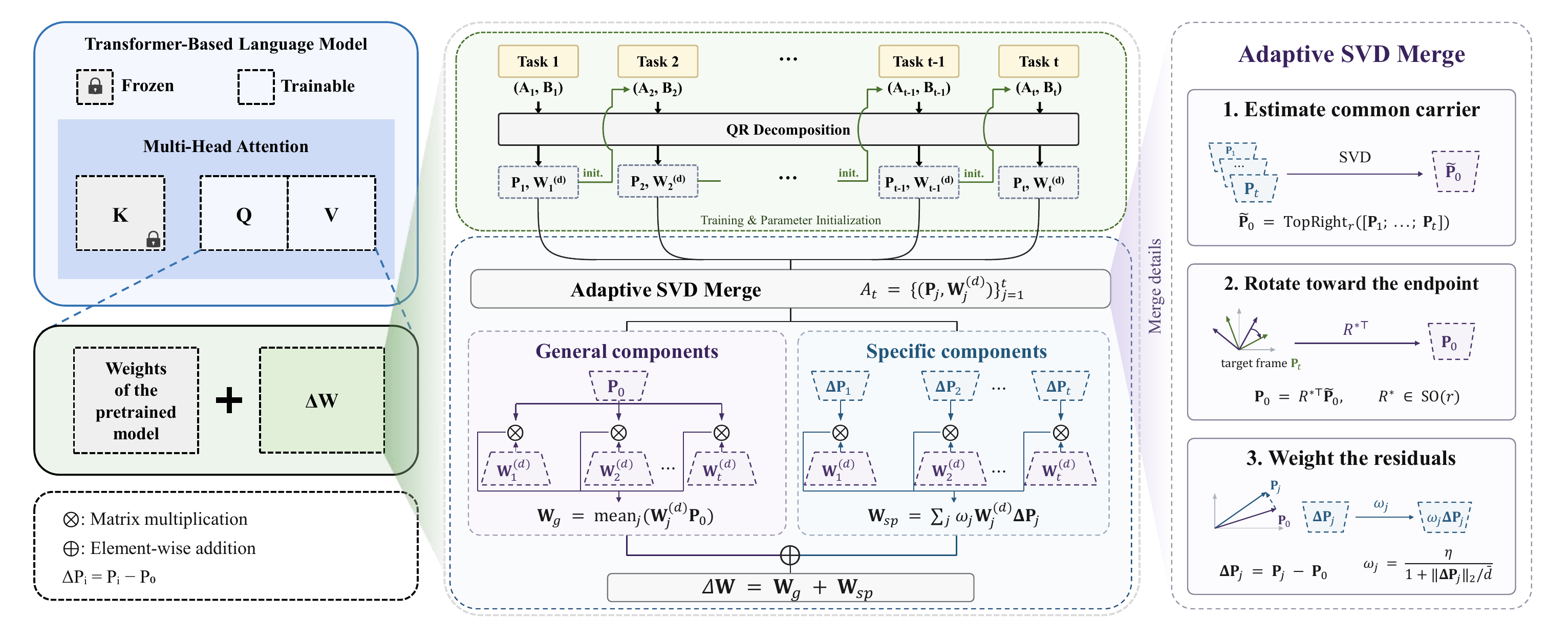}
	\caption{Overview of ChainLoRA. \emph{During} the stream, each update is split by thin QR into a row-orthonormal carrier \(\mathbf P_i\) and magnitude \(\mathbf W_i^{(d)}\); the next task is initialized from the last carrier, and a one-sided proxy loss encourages disentanglement of common knowledge from that of the task-specific knowledge. \emph{After} the stream, Adaptive SVD estimates a shared rank-\(r\) carrier, rotates it in \(\mathrm{SO}(r)\) --- the group of \(r\)-dimensional rotations --- onto the latest-task carrier; residuals are then weighted by their distance from that, as described in Appendix~\ref{app:residual_weights}.}
	\label{fig:framework}
\end{figure*}

Setting replay-based approaches aside, we focus on
\emph{replay-free} continual adaptation over \emph{long} task streams. As tasks accumulate, maintaining previously acquired knowledge while accommodating new tasks becomes increasingly challenging under a limited parameter budget. We consider two representative strategies: orthogonality-based adaptation, exemplified by O-LoRA~\citep{wang2023orthogonal}, and architectural isolation through task-specific modules and routing. These strategies motivate two closely related challenges.

The first challenge is \emph{growing inter-task interference}.
As illustrated in Fig.~\ref{fig:inter_task}, each incoming task
may interact with an increasing number of historical tasks,
creating more opportunities for conflicting updates. OLoRA~\citep{wang2023orthogonal} mitigates such interference by separating task-update directions. However, orthogonality only provides an optimization objective without explicitly addressing how multiple task updates should be merged. Moreover, constraining each new update against previous ones incurs growing training costs. Architectural isolation instead assigns tasks to different modules. Inter-task interference may persist when multiple tasks are routed to shared components with limited parameter capacity.

\begin{wrapfigure}{r}{0.4\textwidth} % r右侧，可改l
    \centering
    \includegraphics[width=\linewidth]{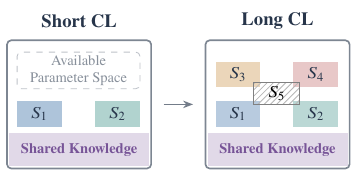}
    \caption{Parameter-space crowding in long-horizon continual
learning: short task streams (left) versus long task streams (right). $S_5$ leads to forgetting.}
    \vspace{-10pt}
    \label{fig:inter_para}
\end{wrapfigure}
The second challenge is \emph{efficient use of limited
adaptation capacity} as Fig.~\ref{fig:inter_para}. OLoRA enforcing orthogonality between entire task updates can be overly restrictive, as different tasks often share transferable knowledge. In the extreme, even splitting the same task into two sequential stages would force their updates into separate directions, unnecessarily consuming additional parameter space. Such over-separation reduces the capacity available for long-term adaptation as more tasks arrive. Architectural isolation faces a similar limitation: allocating new parameters to different tasks primarily postpones capacity saturation through model expansion, rather than improving the reuse of existing capacity. Meanwhile, isolating tasks in separate modules may also limit cross-task generalization.

We treat that question as a problem of \emph{task-vector geometry} rather than of allocating extra modules. Building upon orthogonalization, we further decompose each task update into a shared component and a task-specific component, instead of enforcing orthogonality between entire task vectors. We encourage the task-specific components to be orthogonal to the shared component and to one another, allowing transferable knowledge to occupy shared directions while isolating task-specific variations. This structured decomposition preserves the benefits of orthogonality without unnecessarily separating knowledge that can be reused across tasks.

A well-adopted observation is that a low-dimensional update can match full-parameter fine-tuning~\citep {wang2023orthogonal}. Write \(\mathbf W=\mathbf W_0+\Delta\mathbf W\) with \(\mathbf W_0\) frozen and \(\Delta\mathbf W=\mathbf B\mathbf A\), so CL knowledge lives in a small additive factor. In a continual task stream, that factor is decomposed into tasks and still couples in two ways: the induced directions of different tasks, and their coefficients can activate those. We specifically standardize each update by a thin QR factorization, taking the row-orthonormal factor as the \emph{task vector} (carrier) and the matching coefficients as magnitude.

The method then acts at two timescales, because the objects available during training and after the stream are not the same. \emph{During} the stream, the post-hoc merge carrier is unavailable: only the last task's carrier can be copied forward. Chain-updated initialization starts the next task from the latest. A one-sided proxy loss encourages disentanglement of common knowledge from that of the task-specific knowledge, while a smoothed group penalty on magnitude columns encourages lower-rank updates, preserving adaptation capacity for future tasks and mitigating potential parameter-space crowding. \emph{After} the stream, Adaptive SVD merging estimates a rank-\(r\) shared carrier, aligns to the latest task by Procrustes adaptation, and damps residuals according to their distance from it. Finally, the merged update is absorbed into a single adapter for inference. Highlights of our contribution are:

$\bm{(1)}$ We cast catastrophic forgetting as an operational geometric quantity in merging: each update indicates a carrier with magnitude, naive merging couples them through inner products, and orthogonality is sufficient for weight-space disentanglement. $\bm{(2)}$ On this geometry, we develop ChainLoRA, combining chain-updated training with a post-stream Adaptive SVD merge. The online penalty and initialization use only the last carrier, so historical cost stays constant in the CL stream. $\bm{(3)}$ We provide theoretical analysis showing that Procrustes adaptation facilitates geometric approximate separation of shared and task-specific components, while the one-sided proxy further bounds inter-task interference. $\bm{(4)}$ Experiments indicate strong replay-free performance. On the long-horizon Large stream, ChainLoRA is the strongest among published replay-free LoRA and merging baselines we compare; on shorter Standard CL it approaches replay-augmented methods such as SAPT-LoRA.

\section{Related Work}
\subsection{Continual Learning and Catastrophic Forgetting}
Continual learning studies how a model can learn a sequence of tasks while mitigating \emph{catastrophic forgetting}. Existing methods fall into three families. \textbf{Regularization-based methods} \citep{kirkpatrick2017overcoming,lopez2017gradient,farajtabar2020orthogonal} constrain updates via importance-weighted penalties to preserve past knowledge, but strong constraints can hinder plasticity. \textbf{Replay-based methods} \citep{rebuffi2017icarl,rolnick2019experience,buzzega2020dark,de2019episodic} mix old or pseudo samples with new data; for foundation-model post-training, storing historical data is often infeasible due to privacy or licensing, and generative replay adds overhead. \textbf{Architecture-based (structural isolation) methods} \citep{aljundi2017expert,mallya2018packnet,wang2022learning,verma2021efficient} allocate a separate module per task, reducing interference at the cost of more parameters and less sharing. Our work keeps the isolation but realizes it geometrically: task share a carrier when they can, and otherwise remain as residuals inside one LoRA budget.

Within LoRA-based continual learning, O-LoRA enforces orthogonality in low-rank updates \citep{wang2023orthogonal}, InfLoRA introduces interference-aware subspace control \citep{liang2024inflora}, SAPT-LoRA augments replay via synthetic historical samples \citep{zhao2024sapt}, and recent variants such as LoRM \citep{salami2025closedform}, CorDA \citep{yang2024corda}, and MagMax \citep{marczak2024magmax} explore different constraints on low-rank adaptation. They typically optimize per-task adaptation rather than a unified geometry for continual merging.

\subsection{Low-Rank Adaptation and Parameter-Efficient Fine-Tuning}
Large language models can be adapted with far fewer trainable parameters than the full network. LoRA \cite{hu2022lora} parameterizes the update as $\Delta \mathbf W = \mathbf B \mathbf A$ (with frozen $\mathbf W_0$ and rank $r \ll d$), consistent with the observation that fine-tuning often has low intrinsic dimension \cite{li2018measuring}. These methods reduce cost but typically use a single adapter or LoRA per model (or per task), and do not specify a shared-trunk / residual geometry for composing a sequence of updates.

Parallel to PEFT, recent work studies merging independently trained adapters without original data \citep{ilharco2023editing,zhang2023composing,yadav2023tiesmerging,ortiz2023task,wortsman2022model}, including LoRA-LEGO's structured recombination \citep{zhao2025merging}, KnOTS's SVD-oriented task arithmetic \citep{stoica2025model}, and OPCM for streaming updates \citep{tang2025merging}. Most existing formulations are not specifically designed for asymmetric LoRA factor roles under continual adaptation, which motivates our chain-updated geometry-aware merging design.

\section{Rethinking Catastrophic Forgetting Through Task Vectors}
\label{sec:method}
Inspired by recent findings, e.g., FlyLoRA \citep{zouflylora}, that inter-task orthogonality improves parameter merging, we revisit continual adaptation as the organization of task updates. 

\subsection{Task Vectors and Interaction Bounds}
Given pretrained weights \(\mathbf W_0\), let the rank-\(r\) update induced by task \(\mathcal T_i\) be \(\Delta\mathbf W_i\in\mathbb R^{n\times m}\). We use a thin QR factorization of the transpose,
\[
\Delta\mathbf W_i^{\top}=\mathbf Q_i\mathbf R_i,
\qquad
\mathbf Q_i\in\mathbb R^{m\times r},\quad
\mathbf R_i\in\mathbb R^{r\times n},\quad
\mathbf Q_i^{\top}\mathbf Q_i=\mathbf I_r.
\]

\begin{definition}[Task Vector]
\label{def:task_vector}
Define
\(
\mathbf P_i=\mathbf Q_i^{\top}\in\mathbb R^{r\times m}
\)
and
\(
\mathbf W_i^{(d)}=\mathbf R_i^{\top}\in\mathbb R^{n\times r}.
\)
Then \(\mathbf P_i\mathbf P_i^{\top}=\mathbf I_r\) and
\(\Delta\mathbf W_i=\mathbf W_i^{(d)}\mathbf P_i\).
We call the row-orthonormal carrier \(\mathbf P_i\) the task vector of \(\mathcal T_i\).
\end{definition}

Under the standard QR sign convention (e.g., constraining diagonal entries of \(\mathbf R_i\) to be positive), this decomposition is unique. Therefore, \(\mathbf P_i\) is strongly task-dependent and summarizes dominant task-induced variation directions, while \(\mathbf W_i^{(d)}\) encodes the corresponding magnitudes. The carrier cross-Gram matrix \(\mathbf P_i\mathbf P_j^{\top}\) measures subspace overlap.

\begin{theorem}
\label{thm:vector_orthogonality}
For two tasks \(i,j\), let
\(\Delta\mathbf W_i=\mathbf W_i^{(d)}\mathbf P_i\) and
\(\Delta\mathbf W_j=\mathbf W_j^{(d)}\mathbf P_j\), with row-orthonormal carriers. Then
\[
\left|\langle\Delta\mathbf W_i,\Delta\mathbf W_j\rangle_F\right|
\le
\left\|\mathbf W_i^{(d)\top}\mathbf W_j^{(d)}\right\|_F
\,\cdot\,
\left\|\mathbf P_j\mathbf P_i^{\top}\right\|_F.
\]
Consequently, \(\mathbf P_i\mathbf P_j^{\top}=\mathbf 0\) is sufficient for
\(\langle\Delta\mathbf W_i,\Delta\mathbf W_j\rangle_F=0\). The proof is given in Appendix~\ref{app:task_vector_bound}.
\end{theorem}

\paragraph{Catastrophic Forgetting from a Task-Vector View.}
Prior work on task arithmetic already studies how superposing task updates produces forgetting: when the induced directions are disentangled, adding the updates is approximately an independent superposition, whereas entangled directions let one task's update change another's behavior~\citep{ortiz2023task}. Naive merging is exactly this superposition without considering task entanglement, \(\Delta\mathbf W_{\mathrm{merge}}=\sum_i\Delta\mathbf W_i\)~\citep{ilharco2023editing}. Theorem~\ref{thm:vector_orthogonality} shows that the cross terms vanish whenever \(\mathbf P_i\mathbf P_j^{\top}=\mathbf 0\), and otherwise grow with both subspace overlap and coefficient coupling. Mutual orthogonality of the task vectors is therefore sufficient to remove this entanglement from the overall weight update.

\subsection{Adaptive SVD Linear Merging}
\label{sec:linear_merging}
\paragraph{Linear Merging.} Prior work has highlighted the linearity of model fusion \citep{tang2024parameterefficient}. If coupling is the problem, a merge should isolate a shared term from task residuals rather than add every update in place. A direct sum of LoRA updates mixes transferable directions ($\mathbf W_g$) with per-task ones ($\mathbf W_{\mathrm{sp}}$). Thus, we merge the QR-standardized archive with a shared-residual rule
\begin{equation}
\Delta \mathbf W
=\mathbf W_g+\mathbf W_{\mathrm{sp}}
=\Bigl(\tfrac1T\sum_{i=1}^{T}\mathbf W_i^{(d)}\Bigr)\mathbf P_0
+\sum_{i=1}^{T}\omega_i\mathbf W_i^{(d)}(\mathbf P_i-\mathbf P_0).
\label{eq:merge}
\end{equation}
The merged weights consolidates coefficients on a common carrier \(\mathbf P_0\); the residual term expresses deviations, without orthogonality assumption. Given archived carriers, the map is linear in \(\{\mathbf W_i^{(d)}\}\). Each \(\omega_i\), and the split of \(\mathbf W_g\) from \(\mathbf W_{\mathrm{sp}}\), is set in Appendix~\ref{app:residual_weights}.

\paragraph{SVD specifies the carrier.} We concatenate all task vectors, apply singular value decomposition (SVD). The rank-\(r\) SVD of the stacked carriers is the best coverage of archived directions.
\[
\mathbf X_T
=[\mathbf P_1^\top,\dots,\mathbf P_T^\top]
=\mathbf U\mathbf\Sigma\mathbf V^\top,
\qquad
\widetilde{\mathbf P}_0=\mathbf U_r^\top,
\]

\paragraph{Procrustes adaptation.}
Before the incoming task \(\mathcal T_{T+1}\), a proper rotation \(\mathbf R_*\in\mathrm{SO}(r)\) aligns the current $\mathbf P_0$ to the last-task-ended carrier, so \(\mathbf P_0\) is expressed in the coordinates of \(\mathbf P_T\) rather than in an arbitrary SVD basis.
\begin{equation}
\mathbf P_0=\mathbf R_*^{\top}\widetilde{\mathbf P}_0,
\qquad
\mathbf R_*\in\arg\min_{\mathbf R\in\mathrm{SO}(r)}
\|\mathbf R^{\top}\widetilde{\mathbf P}_0-\mathbf P_T\|_F^2,
\label{eq:adapt_svd_alignment}
\end{equation}
A rank-\(r\) SVD specifies a covering subspace, not a unique coordinate frame --- \(\widetilde{\mathbf P}_0\) can be replaced by any \(\mathrm{SO}(r)\) rotation of itself. The only sequentially meaningful index is the task just completed: that endpoint \(\mathbf P_T\) is the carrier we have just obtained and are able to keep its high performance. Then, we seek earlier carriers orthogonal to it to preserve former tasks' performance. In Section~\ref{sec:chain}, we build \textit{task geometry} for that. Thm.~\ref{thm:similar_carriers} shows that this adaptation mitigates forgetting under the geometry.

\begin{theorem}
\label{thm:similar_carriers}
Let \(\mathbf P\) be a common carrier after Procrustes adaptation and \(\mathbf Q\) a just-learned task vector. Their difference is \(\mathbf E=\mathbf Q-\mathbf P\), which is also the task-specific knowledge. Denote the loss by  
\begin{equation}
\mathcal L_g^i
=\frac{\|\mathbf P\mathbf E^{\top}\|_F^2}{\|\mathbf E\|_F^2},
\qquad\mathbf E\neq\mathbf 0.
\label{eq:lg_def}
\end{equation}
Procrustes alignment minimizes \(\|\mathbf E\|_F\le\eta\), which implies \(\mathcal L_g^i\le\eta^2/4\), supporting an approximate geometric separation between task-specific knowledge and common carrier. Proof in Appendix~\ref{app:stiefel_loss}.
\end{theorem}

\section{Chain-updated Learned Task Geometry}
\label{sec:chain}
The merge above assumes a finished archive.
Continual learning must produce that archive online: each task's data is
available only during its own training stage, with \textit{no replay} of earlier
tasks. Three choices follow: \textit{first}, the next carrier is initialized from the
previous one, so training starts at the inherited carrier rather than from a random basis; \textit{second}, a one-sided proxy regularizes displacement of the
current task vector from the last-task carrier, which caches a \emph{detached} copy
of the current merge and adds it (stop-grad) on the next task's unmerged
forward; \textit{third}, after each task the pair
\((\mathbf P_t,\mathbf W_t^{(d)})\) is archived for the eventual merge.

\begin{algorithm}[h]
\caption{ChainLoRA}
\label{alg:chainlora}
\begin{algorithmic}[1]
\REQUIRE Base model weights $\mathbf W_{0}$; tasks $\{\mathcal{T}_t\}_{t=1}^{T}$;
\STATE Initialize LoRA: kaiming\_uniform($\mathbf A_1$), $\mathbf B_1=0$; bake cache empty.
\STATE \textbf{Function} $\mathrm{GetProjAndCoeff}(\mathbf A,\mathbf B)$:
\STATE \quad $\mathbf A^\top=\mathbf Q\mathbf R,\ 
\mathbf S\gets\operatorname{diag}\!\bigl(\operatorname{sign}(\operatorname{diag}(\mathbf R))\bigr),\ 
\widehat{\mathbf Q}\gets\mathbf Q\mathbf S,\ \widehat{\mathbf R}\gets\mathbf S\mathbf R$
\STATE \quad $\mathbf P\gets\widehat{\mathbf Q}^\top,\ {\mathbf W}^{(d)}\gets\mathbf B \widehat{\mathbf R}^\top$
\STATE \quad \textbf{return} $(\mathbf P, {\mathbf W}^{(d)})$
\FOR{$t = 1$ to $T$}
    \STATE Train on task $\mathcal{T}_t$, obtain: $\mathbf{A}_{t}$, $\mathbf{B}_{t}$. \\
    \STATE $(\mathbf P_t, {\mathbf W_t}^{(d)}) \gets \mathrm{GetProjAndCoeff}(\mathbf{A}_{t}, \mathbf{B}_{t})$
    \STATE Merge the current archive for evaluation; if bake is enabled, cache its detached update for next-task training.
    \STATE Initialize $\mathbf A_{t+1}\gets\mathbf P_t$, $\mathbf B_{t+1}\gets 0$
\ENDFOR
\STATE For $T>1$, estimate $\mathbf P_0$ by SVD and Eq.~\eqref{eq:adapt_svd_alignment}.
\STATE Compute $\omega_i$, $\mathbf W_g$ and $\mathbf W_{\mathrm{sp}}$ (Appendix~\ref{app:residual_weights}).
\STATE $\Delta\mathbf W\gets (\alpha/r)\mathbf W_g+\mathbf W_{\mathrm{sp}}$; for $T=1$, use $\mathbf B_1\mathbf A_1$.
\RETURN Final model weights $\mathbf W_0+\Delta \mathbf W$
\end{algorithmic}
\end{algorithm}

\paragraph{Chain-updated initialization.}
Task \(1\) starts from a random \(\mathbf P_1\) and \(\mathbf W_1^{(d)}=\mathbf 0\). For \(i\ge 2\) we set \(\mathbf P_i\leftarrow\mathbf P_{i-1}\) and \(\mathbf W_i^{(d)}\leftarrow\mathbf 0\), then optimize raw factors \(\{\mathbf P_i,\mathbf W_i\}\) on task \(i\) and archived them. The next carrier therefore begins at the previous one.

\begin{theorem}[]
\label{thm:normalized_interference}
\label{eq:surrogate_bound}
Let \(\mathbf E_i=\mathbf P_i-\mathbf P_{i-1}\neq\mathbf 0\), let each earlier increment \(\mathbf E_j\neq\mathbf 0\), and write
\(\widehat{\mathbf E}_k=\mathbf E_k/\|\mathbf E_k\|_F\).
Assume \(\mathbf H_{i-1}\) is row-orthonormal. For every \(j<i\), define the new trunk-growth overlap and the earlier-growth
trunk gap
\[
\epsilon_i
=\bigl\|\mathbf H_{i-1}\widehat{\mathbf E}_i^{\top}\bigr\|_F,
\qquad
\delta_{j\mid i-1}
=\bigl\|\widehat{\mathbf E}_j(\mathbf I-\mathbf\Pi_{i-1})\bigr\|_F,
\]
and the corresponding normalized trunk loss
\begin{equation}
\mathcal L_{g}^{i}
=\frac{\|\mathbf H_{i-1}\mathbf E_i^{\top}\|_F}{\|\mathbf E_i\|_F}
=\epsilon_i.
\end{equation}
Then
\begin{equation}
\bigl\|\widehat{\mathbf E}_j\widehat{\mathbf E}_i^{\top}\bigr\|_F
\le
\mathcal L_{g}^{i}
+\delta_{j\mid i-1}.
\label{eq:interference_bound}
\end{equation}
\end{theorem}

\paragraph{Chain-updated one-sided proxy.}
Before task \(i\), the post-stream reference \(\mathbf P_0\) is
unavailable. The inherited carrier \(\mathbf H_{i-1}=\mathbf P_{i-1}\)
is therefore the only trunk the trainer can use. Therefore, we use \(\mathbf H_{i-1}\) as an effective surrogate.
\begin{equation}
\mathcal L_{g}^i=\sum_{i=1}^{T}\left\|\mathbf P_0\Delta \mathbf P_i^\top\right\|_F^2
\approx \sum_{i=2}^{T}\left\|\mathbf H_{i-1}(\mathbf P_{i}-\mathbf P_{i-1})^\top\right\|_F^2
\label{eq:shared_orth_loss}
\end{equation}

To reason about interference along this chain, write
\(\mathbf E_i=\mathbf P_i-\mathbf P_{i-1}\) for the local change of the
carrier after task \(i\), which is a main carrier rotation. Before task \(i\), \(\mathbf H_{i-1}\) passes a fixed-rank historical information. By Theorem~\ref{eq:surrogate_bound}, we could imply that: the orthogonality term of $\Delta \mathbf P_i\Delta \mathbf P_j^\top,\ \forall i\neq j$ could also be minimized by $\mathcal L_{g}^i$. The proof is given in Appendix~\ref{app:normalized_interference}.

\section{When LoRA Meets Task Geometry}
\label{sec:lora}
We now instantiate the geometry with LoRA factors: how \(\mathbf P_i\) and \(\mathbf W_i^{(d)}\) are computed, what is optimized online, and how the archive is merged. The overall procedure is Algorithm~\ref{alg:chainlora}. 

\subsection{\texorpdfstring{LoRA Corresponds of \(\mathbf P_i\) and \(\mathbf W_i^{(d)}\)}{LoRA Corresponds of Pi and Wi}}
\label{sec:lora_correspondence}
For task \(i\), we use LoRA factors \(\mathbf A_i\in\mathbb R^{r\times m}\), \(\mathbf B_i\in\mathbb R^{n\times r}\), and define $\Delta\mathbf W_i=\mathbf B_i\mathbf A_i.$
To align this with \(\Delta\mathbf W_i=\mathbf W_i^{(d)}\mathbf P_i\), we perform thin QR on \(\mathbf A_i^\top\):
\[
\mathbf A_i^\top=\mathbf Q_i\mathbf R_i,\quad
\mathbf S_i=\operatorname{diag}\!\bigl(\operatorname{sign}(\operatorname{diag}(\mathbf R_i))\bigr),\quad
\widehat{\mathbf Q}_i=\mathbf Q_i\mathbf S_i,\quad
\widehat{\mathbf R}_i=\mathbf S_i\mathbf R_i.
\]
Then we have $\mathbf P_i=\widehat{\mathbf Q}_i^\top$, $\mathbf W_i^{(d)}=\mathbf B_i\widehat{\mathbf R}_i^\top$. So that \(\Delta\mathbf W_i=\mathbf W_i^{(d)}\mathbf P_i\) holds exactly.

\paragraph{Improving Parameter Efficiency by Effective-Rank Reduce.} We also constrain the Lora parameters with an explicit low-rank surrogate. Since each column of $\mathbf W^{(d)}$ and each row of $\mathbf P$ corresponds to one latent adaptation direction, shrinking column norms encourages some directions to vanish, thereby reducing the effective rank and minimizing perturbation to the pretrained backbone. Since \(\mathbf P\) is an orthogonal matrix, its contribution becomes a constant term and does not affect optimization; therefore, the final loss only contains terms \(w\).
\[
\Delta \mathbf W_i =\mathbf P_i \mathbf W_i^{(d)} = \sum_{k=1}^{r} p_k w_k ^\top,\quad
\mathcal L_{\mathrm{rank}}^i
= \sum_{k=1}^{r}\left(\|p_k\|_2^2 + \|w_k\|_2^2\right) = \sum_{k=1}^{r}\|w_k\|_2^2
\]

\paragraph{Initialization.}
\(\mathbf A_1\sim\mathrm{Kaiming}\) and \(\mathbf B_1=\mathbf 0\). For \(i\ge2\), we use $\mathbf A_i\leftarrow \mathbf P_{i-1},\,
\mathbf B_i\leftarrow \mathbf 0.$

\paragraph{Per-task total loss.}
At step of training task \(i\), we optimize $\mathcal L^{i}=\mathcal L_t^i+\lambda_1\mathcal L_g^i+\lambda_2\mathcal L_{\mathrm{rank}}^i$. After all tasks are learned, updates are merged using Eq.~\ref{eq:merge}.

\paragraph{Merging Phase.} According to Eq.~\ref{eq:merge}, we merge the previous tasks parameters and gets the final $\Delta \mathbf W$. A key point for merging take effect is the chain-update task geometry. That orthogonality of tasks induced by $\mathcal L_g^i$ lets the merge take effect, and chain-updated initialization carries forward the common carrier. The merge depends on both; without them it need not work.

\paragraph{Remarks.} Moreover, our formulation can be viewed as a strict generalization of OLoRA-style merging. Specifically, Eq.~\ref{eq:merge} explicitly decomposes updates into a shared component (through \(\mathbf P_0\)) and a task-specific residual component. When \(\mathbf P_0=\mathbf 0\), the shared term vanishes and Eq.~\ref{eq:merge} reduces to $\Delta \mathbf W=\sum_{i=1}^{T}\mathbf P_i\mathbf W_i^{(d)}=\sum_{i=1}^{T}\mathbf A_i\mathbf B_i$, which recovers the OLoRA-like merge pattern that only aggregates task-specific projected updates.

\section{Experiments}
\subsection{Experimental Setup}
\noindent\textbf{\textit{Base Models.}}~We evaluate our approach across three models: Llama-2-7B-chat, Llama-2-13B-chat, and Llama-3.2-3B. The 3B checkpoint differs from Llama-2 in family as well as size. Although this comparison does not control variables in the strictest sense, it still shows that the method transfers across model families and across sizes. All experiments were conducted on four RTX~4090 GPUs. The learning rate is $2\times10^{-4}$. Each task in the stream is trained for two epochs by LoRA with bake training as in Appendix~\ref{app:bake}. The geometry loss is extended in multi-layers as in Appendix~\ref{app:runtime_objective}.\\
\noindent\textbf{\textit{Datasets.}}~We consider three continual learning benchmarks: (1) \textit{Standard CL benchmark}: AG News, Amazon Reviews, DBpedia, and Yahoo Answers. (2) \textit{Large number of tasks}: five Standard CL benchmark tasks, four GLUE tasks (MNLI, QQP, RTE, SST-2), five SuperGLUE tasks (WiC, CB, COPA, MultiRC, BoolQ), and IMDB movie reviews. Following {\sffamily{O-LoRA}}~\citep{wang2023orthogonal}, each task uses 1000 randomly sampled training samples and 500 validation samples per class. (3) \textit{SuperNI Benchmark}~\citep{wang2022super}: SuperNI is an instruction-following mix of generation and classification, unlike the classification streams that follow. A diverse collection of NLP tasks with expert-written instructions, covering dialogue generation, information extraction, question answering, summarization, and sentiment analysis. We follow task selection and ordering in {\sffamily{SAPT}}~\citep{zhao2024sapt}, using 1,000 training instances and 100 for validation/testing per task.\\
\noindent\textbf{\textit{Baselines.}}~We compare our method with the following baselines: (1) \textit{Continual learning baselines:} {\sffamily{SeqLoRA}}: sequentially fine-tunes a single LoRA on multiple tasks without constraints; {\sffamily{IncLoRA}}: incrementally adds a new LoRA per task while freezing previous LoRAs; {\sffamily{O-LoRA}}~\citep{wang2023orthogonal}; {\sffamily{InfLoRA}}~\citep{liang2024inflora}; {\sffamily{SAPT-LoRA}}~\citep{zhao2024sapt};  {\sffamily{LoRM}}~\citep{salami2025closedform}; {\sffamily{CorDA}} (knowledge-preserved adaptation)~\citep{yang2024corda}; {\sffamily{MagMax}}~\citep{marczak2024magmax}. (2) \textit{LoRA merging baselines:} {\sffamily{LoRA-LEGO}}~\citep{zhao2025merging}; {\sffamily{KnOTS}}~\citep{stoica2025model}. (3) \textit{Continual merging baseline:} {\sffamily{OPCM}}~\citep{tang2025merging}. % \\
\begin{table}[H]
  \caption{Final average task scores (\%) for Llama-2-7B-chat ($Oi$=order). Only the replay-free method is compared. Methods are listed in the order described above, while the two merging strategies of LORM are listed separately.}
  \label{tab:llama7b_slao_overall}
  \centering
  \resizebox{\textwidth}{!}{
    \begin{tabular}{c|c c c c|c c c c| c c c}
    \toprule
        & \multicolumn{4}{c|}{\textbf{Standard CL Benchmark}} 
         & \multicolumn{4}{c|}{\textbf{Large Number of Tasks}} & \multicolumn{3}{c}{\textbf{SuperNI Benchmark}} \\
        \midrule
        \textbf{Method} & \textbf{O1} & \textbf{O2} & \textbf{O3} & \textbf{avg} 
        & \textbf{O4} & \textbf{O5} & \textbf{O6} & \textbf{avg} & \textbf{O1} & \textbf{O2} & \textbf{avg}\\
        \midrule
        SeqLoRA&73.3&76.2&78.4&76.0&69.1&66.0&71.1&68.7&18.4&26.8&22.6\\
        IncLoRA&75.3&77.3&78.3&77.0&72.2&71.6&73.8&72.5&22.0&25.6&23.8\\
        O-LoRA & 76.1 & 76.3 & 79.2 &77.2& \underline{74.0} & 72.0& \underline{74.6}&\underline{73.5}& 23.3&28.4&25.9\\
        InfLoRA & 78.4 & \underline{80.4} & 79.9 & 79.6& 69.4& 67.4&72.5&69.8&16.5 & 22.1 & 19.3\\
        \midrule
        LoRM-BA&76.0&76.8&78.3& 77.0&71.4&69.0&70.3&70.2&25.6&18.7&22.2\\
        LoRM-AB&77.5&74.7&75.9& 76.0& 71.0&69.5&70.2&70.2&25.6&23.7&24.7\\
        \midrule
        CorDA & 78.4&79.3&80.0&79.2&73.4&72.7&74.0&73.4&20.9&16.0&18.5\\
       MagMax &\underline{80.1} & \textbf{80.6}& \textbf{80.3}& \textbf{80.3}& 72.3&\underline{73.5}&74.5&73.4&15.3&7.0&11.2\\
        \midrule
        OPCM &61.9&62.0&56.7&60.2&51.9&52.8&46.9&50.5& 11.6&12.3&12.0\\
        \midrule
        KnOTS&67.9&65.9&70.8&68.2 &61.5&60.1&58.0&59.9&\textbf{34.6}&\underline{30.1}&\underline{32.4}\\
        LoRA-LEGO& 68.3 & 66.0& 70.9 &68.4 &58.8 &58.7 & 53.2&56.9 & \underline{32.8}&26.7&29.8\\
        \rowcolor{gray!20}
        ChainLoRA (ours) &\textbf{80.5} &79.6 &\underline{80.2} & \underline{80.1} &\textbf{77.3} &\textbf{75.8} &\textbf{74.6}&\textbf{75.5} &29.8 &\textbf{37.8} &\textbf{33.8} \\
        \midrule
        (Replayed) SAPT& 82.9&81.8&78.7&81.1& 84.7&78.9&82.2&81.9&53.2&48.5&50.9\\
       \bottomrule
    \end{tabular}
    }
\end{table}
\noindent\textbf{Evaluation metrics. }Every suite reports average accuracy (AA) $\frac{1}{T}\sum_{i=1}^{T}a_{i,T}$ after the last task. Classification sets $a_{i,T}$ by exact match; SuperNI is evaluated by Rouge-L first, and AA is the average of those scores. Backward transfer (BWT) $\frac{1}{T-1}\sum_{i=1}^{T-1}(a_{i,T}-a_{i,i})$, is recorded in Appendix~\ref{app:additional_experiments}.

\subsection{Overall Results}
On the short Standard stream, our method approaches the best results alongside MagMax and InfLoRA, and replay-augmented SAPT-LoRA leads only slightly. The replay lead is much clearer on long streams: visible on Large, and largest on SuperNI, where SAPT sits far above the replay-free rows. Even there, ChainLoRA is closest to SAPT, which further shows that the method is strong.

\begin{wrapfigure}{r}{0.48\textwidth}
    \centering
    \vspace{-15pt}
    \includegraphics[width=\linewidth,height=0.12\textheight,keepaspectratio]{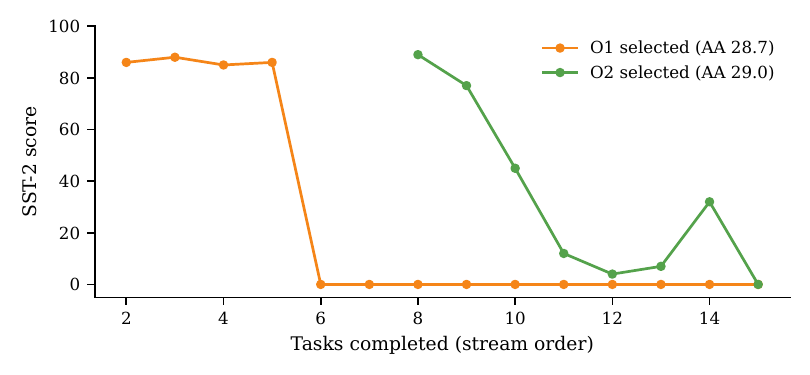}
    \caption{SuperNI SST-2 collapsed.}
    \label{fig:sni_sst2}
    \vspace{-20pt}
\end{wrapfigure}
SST-2 is a task in the SuperNI stream. It is learned, then not retained: after later tasks, the score falls to 0 (Figure~\ref{fig:sni_sst2}). On O1 it appears early, at position 2, and is already 0 by task 6. On O2 it appears at position 8 and finishes at 0 on the last task. ChainLoRA still sits above the listed merging baselines (KnOTS, LoRA-LEGO). We also believe this task is the opening for our follow-up work.

\subsection{Ablation Study}
\vspace{-1.5mm}

\textbf{Effect of each module.}
Table~\ref{tab:component_ablation_matched} asks which piece matters on which suite. On the short Standard stream only a few tasks are merged, so the merge rule moves the score the most: removing Adaptive SVD drops AA by $5.0$. Longer streams already carry some forgetting, and the same removal costs less, $2.3$ on Large and $1.3$ on SuperNI. The online losses go the other way. They stay small while the tasks are short classification problems, then grow once the stream is longer and the tasks are more mixed: on SuperNI, removing $\mathcal L_g^i$ costs $6.2$ and removing the rank loss costs $6.3$. Keeping chain-updated initialization is necessary on that stream as well. Without it, SuperNI falls by $3.1$.  
\begin{table}[H]
    \caption{Ablation study of each modules as Table~\ref{tab:llama7b_slao_overall}.}
  \label{tab:component_ablation_matched}
  \centering
  \resizebox{\textwidth}{!}{
  \begin{tabular}{lcccccc}
    \toprule
    \textbf{Variant}
      & \textbf{Standard CL avg} & $\Delta$
      & \textbf{Large avg} & $\Delta$
      & \textbf{SuperNI avg} & $\Delta$ \\
    \midrule
    \rowcolor{gray!20}
    Full (Table~\ref{tab:llama7b_slao_overall} recipe)
      & 80.1 & --
      & 75.5 & --
      & 33.8 & -- \\
    w/o $\mathcal L_g^i$ ($\lambda_1{=}0$)
      & 79.3 & $-0.8$
      & 74.4 & $-1.1$
      & 27.6 & $-6.2$ \\
    w/o $\mathcal L_{\mathrm{rank}}^i$ ($\lambda_{2}{=}0$)
      & 79.0 & $-1.1$
      & 75.4 & $-0.1$
      & 27.5 & $-6.3$ \\
    w/o Adaptive SVD Merging
      & 75.1 & $-5.0$
      & 73.2 & $-2.3$
      & 32.5 & $-1.3$ \\
    w/o Chain-updated initialization
      & 78.2 & $-1.9$
      & 74.8 & $-0.7$
      & 30.7 & $-3.1$ \\
    \bottomrule
  \end{tabular}
  }
\end{table}

\paragraph{Varying the merge.}
Table~\ref{tab:assymetry} compares three ways to compose the same archive. Average adds the task updates with equal weight. SVD keeps only the rank-$r$ carrier and leaves that basis where the decomposition put it. Adaptive SVD is the full rule: the same carrier, then rotated onto $\mathbf P_T$ by Procrustes. SVD without that alignment collapses. Standard falls to $55.6$, SuperNI to $11.3$, and Large O4/O5 finish at $0$. The SVD factor is a covering subspace, not the coordinate frame of $\mathbf P_T$. Without that alignment, carrying the basis straight into the merge makes the result change with the task order, which continual learning cannot accept. Average already keeps a usable score ($65.0$ / $67.6$ / $28.8$). Different tasks move in different directions in the update, so a direct sum still carries each of them; the geometric loss further puts those updates in a nearly additive state, which is why the naive sum works. Adaptive SVD remains higher row across the three suites.

\begin{table}[H]
  \caption{Merging ablation study. $\dagger$ marks results that ended in catastrophic forgetting (AA $=0$).}
  \label{tab:assymetry}
  \centering
  \resizebox{\textwidth}{!}{
    \begin{tabular}{c|c c c c|c c c c| c c c}
    \toprule
         & \multicolumn{4}{c|}{\textbf{Standard CL Benchmark}}
         & \multicolumn{4}{c|}{\textbf{Large Number of Tasks}} & \multicolumn{3}{c}{\textbf{SuperNI Benchmark}} \\
         \midrule
      \textbf{Merging} & \textbf{O1} & \textbf{O2} & \textbf{O3} & \textbf{avg}
        & \textbf{O4} & \textbf{O5} & \textbf{O6} & \textbf{avg} & \textbf{O1} & \textbf{O2} & \textbf{avg}\\
        \midrule
        Average & 74.2& 65.8& 55.0&65.0& 67.6&66.7& 68.4&67.6 &28.4&29.2&28.8\\
        SVD (without Procrustes) &41.6&64.3&60.8&55.6&$\dagger$&$\dagger$&54.3&--&10.4&12.2&11.3\\
        \midrule
        \rowcolor{gray!20}
        Adaptive SVD (ours) &80.5 &79.6 &80.2 & 80.1 &77.3 &75.8 &74.6 &75.5&29.8 &37.8 &33.8\\
       \bottomrule
    \end{tabular}
    }
\end{table}

\paragraph{Effect of model size.}
Table~\ref{tab:model_size} lists three base models to check whether the same continual method still works across model family and across size. The comparison does not control every variable in the strictest sense: Llama-3.2-3B is there because it differs from Llama-2 in family as well as size. The method remains effective on all three. Larger size is not uniformly better, since the intrinsic dimension of the model can change as its size grows, and it does not always perform best. At this scale, 7B and 13B still outperform 3B ($80.1$/$75.5$/$33.8$ and $75.5$/$76.6$/$23.7$ against $72.0$/$73.3$/$16.1$). A larger parameter space can hold more tasks thus forget less, which is the gain from 3B up to the Llama-2 models. From 7B to 13B the gain is much smaller, and on the short Standard stream 7B is ahead of 13B ($80.1$ vs.\ $75.5$), as it is on SuperNI ($33.8$ vs.\ $23.7$). The order reverses on the Large datasets, and 13B is ahead of 7B ($76.6$ vs. $75.5$), so the two sizes cross. We attribute the difference to depth: 13B has more layers than 7B; loss averaged over layers might lower a given task performance.

\begin{table}[H]
  \caption{Results across model size 3B/7B/13B.}
  \label{tab:model_size}
  \centering
  \resizebox{\textwidth}{!}{
    \begin{tabular}{c|c c c c|c c c c| c c c}
    \toprule
    & \multicolumn{4}{c|}{\textbf{Standard CL Benchmark}}
         & \multicolumn{4}{c|}{\textbf{Large Number of Tasks}} & \multicolumn{3}{c}{\textbf{SuperNI Benchmark}} \\
         \midrule
      \textbf{Model} & \textbf{O1} & \textbf{O2} & \textbf{O3} & \textbf{avg}
        & \textbf{O4} & \textbf{O5} & \textbf{O6} & \textbf{avg} & \textbf{O1} & \textbf{O2} & \textbf{avg}\\
        \midrule
        Llama-3.2-3B &73.6&70.4&72.1&72.0& 73.5&72.9&73.4&73.3&14.5&17.8&16.1\\
        \midrule
        Llama-2-7B-chat&80.5 &79.6 &80.2 & 80.1&77.3 &75.8 &74.6 &75.5&29.8 &37.8 &33.8\\
        Llama-2-13B-chat&75.3&74.5&76.7&75.5&78.8&73.5&77.6&76.6&26.4&21.1&23.7\\
       \bottomrule
    \end{tabular}
    }
\end{table}
\begin{figure}[tb]
    \centering
    \includegraphics[width=0.30\textwidth]{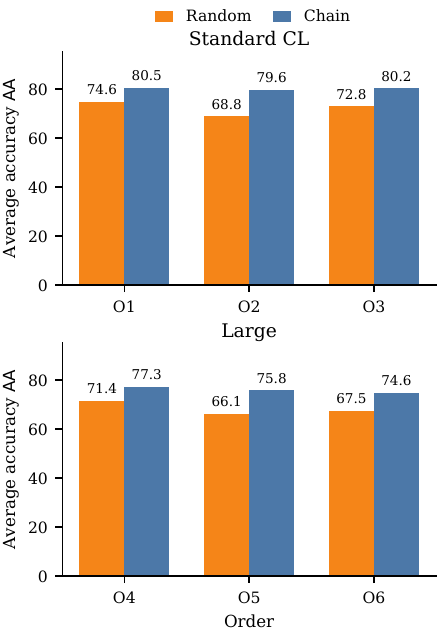}
    \caption{Random vs.\ Chain init on Standard CL and Large.}
    \label{fig:init_ablation}
\end{figure}
\vspace{-5pt}
\paragraph{Varying $\mathbf A_i$ initialization.} Figure~\ref{fig:init_ablation} isolates the chain-updated carrier with random initialization. The latter is already weaker on the short Standard stream; the gap widens on Large, where each new task otherwise starts from an unrelated basis while the archive is still growing. Because the last carrier is the only shared object the next task can see, it gives a better initialization, and both orthogonalization and learning similar tasks.

\vspace{-5pt}
\subsection{Continual learning trajectory}
\vspace{-1.5mm}
Order averages AA hide where forgetting actually happens. Table~\ref{tab:size_per_task} is one Standard order, DBpedia$\to$Amazon$\to$Yahoo$\to$AGNews, so the four rows are the four tasks. The three scores are in different times: $a_{i,0}$ is exact match before task $i$ is seen, $a_{i,i}$ is the score just after that task is trained, and $a_{i,T}$ is the score after the whole stream. On the 3B checkpoint, $a_{i,0}$ stays near zero because the model cannot yet emit the label format without training; only after training is the accuracy high enough for a decent prediction on that task, which is why the next column jumps. What the table shows is that forgetting depends on the task: AGNews barely moves once it is learned, while Yahoo can lose a large amount after learning new tasks, especially at 13B. Figure~\ref{fig:sni_sst2} indicates the same point on SuperNI: SST-2 finishes at $0$ extremely.
\begin{table}[H]
    \caption{Standard CL O1 per-task scores (\%) at three sizes (DBpedia$\to$Amazon$\to$Yahoo$\to$AGNews); $a_{i,0}$ unseen, $a_{i,i}$ just after task $i$, $a_{i,T}$ after the stream, $\Delta{=}a_{i,T}{-}a_{i,i}$.}
  \label{tab:size_per_task}
  \centering
  \resizebox{\textwidth}{!}{
  \begin{tabular}{lcccc|cccc|cccc}
    \toprule
    & \multicolumn{4}{c|}{\textbf{3B}} & \multicolumn{4}{c|}{\textbf{7B}} & \multicolumn{4}{c}{\textbf{13B}} \\
    \textbf{Task} & $a_{i,0}$ & $a_{i,i}$ & $a_{i,T}$ & $\Delta$ & $a_{i,0}$ & $a_{i,i}$ & $a_{i,T}$ & $\Delta$ & $a_{i,0}$ & $a_{i,i}$ & $a_{i,T}$ & $\Delta$ \\
    \midrule
    DBpedia & 1.4 & 86.6 & 84.6 & $-2.0$ & 60.7 & 98.8 & 95.6 & $-3.2$ & 71.4 & 98.7 & 97.2 & $-1.6$ \\
    Amazon & 3.5 & 54.8 & 52.8 & $-2.0$ & 34.4 & 62.5 & 60.0 & $-2.4$ & 44.6 & 54.0 & 59.9 & $+5.9$ \\
    Yahoo & 0.0 & 70.4 & 67.4 & $-3.0$ & 14.9 & 74.3 & 71.2 & $-3.1$ & 34.8 & 74.9 & 54.0 & $-20.9$ \\
    AGNews & 17.0 & 89.3 & 89.4 & $+0.1$ & 61.9 & 90.1 & 90.0 & $-0.1$ & 56.6 & 90.0 & 90.1 & $+0.0$ \\
    \midrule
    AA / mean $\Delta$ on old & \multicolumn{3}{c}{73.6} & $-2.3$ & \multicolumn{3}{c}{79.2} & $-2.9$ & \multicolumn{3}{c}{75.3} & $-5.5$ \\
    \bottomrule
  \end{tabular}
  }
\end{table}

\vspace{-10pt}
\section{Conclusion}
We presented \textbf{ChainLoRA}, a replay-free continual merging framework built on chain-updated task-vector geometry. The resulting linear merge consolidates transferable directions while preserving task-private knowledge. During the stream, chain-updated initialization copies the previous carrier forward, and a one-sided orthogonality surrogate limits drift from that carrier effectively. We further imposed effective-rank regularization to control long-horizon parameter growth for parameter efficiency. The cost stays constant as the stream lengthens. After training, the archive is absorbed into a single adapter for inference. Experiments show that ChainLoRA is the strongest replay-free method we compare on the long stream and approaches comparable performance on the short stream.

\clearpage
\subsection*{AI use statement}
In this work, we used generative AI tools to edit the manuscript for readability and LaTeX formatting. We have not used generative AI tools to generate datasets, experimental numbers, or figures. We reviewed the AI-assisted text and checked the mathematical claims against the appendix. We take responsibility for the final content of this work.

\subsection*{Ethics statement}
This work evaluates publicly available language models on public continual-learning benchmarks. It does not involve human subjects, does not release a new dataset, and does not store user data; the method is replay-free. We are not aware of additional ethics concerns beyond those of ordinary NLP fine-tuning.

\subsection*{Reproducibility statement}
All settings are described in the experimental setup or in the design details of Appendix~\ref{app:setup_details}. The procedure is given as Alg.~\ref{alg:chainlora}.

\bibliographystyle{iclr2027_conference}
\bibliography{references}
\appendix
\clearpage
\appendix
\onecolumn

\newcommand{\AppendixTOCEntry}[3]{%
    \noindent
    \hyperref[#3]{\textbf{#1}\hspace{0.75em}#2}%
    \nobreak\dotfill\nobreak
    \hyperref[#3]{\pageref*{#3}}\par
}
\newcommand{\AppendixTOCSubEntry}[3]{%
    \noindent\hspace*{1.8em}%
    \hyperref[#3]{#1\hspace{0.75em}#2}%
    \nobreak\dotfill\nobreak
    \hyperref[#3]{\pageref*{#3}}\par
}

\phantomsection
\label{app:contents}
\begin{center}
{\Large \bf Appendix}
\end{center}

This appendix is organized in four blocks.
\textbf{Theory}: coefficient-weighted task-vector interaction (Appendix~\ref{app:task_vector_bound}), the normalized interference bound (Appendix~\ref{app:normalized_interference}), and fourth-order overlap with the similar-task bound (Appendix~\ref{app:stiefel_loss}).
\textbf{Training and merge} (Appendix~\ref{app:setup_details}): layer-averaged geometric and rank losses (Appendix~\ref{app:runtime_objective}), inter-task bake (Appendix~\ref{app:bake}), and distance-dependent residual weights (Appendix~\ref{app:residual_weights}).
\textbf{Backward transfer and running curves} (Appendix~\ref{app:additional_experiments}): final backward transfer (Appendix~\ref{app:bwt_results}) and running backward transfer and average accuracy (Appendix~\ref{app:aa_sequence}).

\section*{Contents of the Appendix}
\pdfbookmark[1]{Contents of the Appendix}{appendix.contents}
\newcommand{\appendixentry}[2]{\par\noindent
  \hyperref[#1]{\ref*{#1}\quad #2}\nobreak\dotfill
  \hyperref[#1]{\pageref*{#1}}\par\smallskip}
\par\noindent\textbf{Theory.}
\appendixentry{app:task_vector_bound}{Proof of Theorem~\ref{thm:vector_orthogonality}}
\appendixentry{app:normalized_interference}{Proof of Theorem~\ref{thm:normalized_interference}}
\appendixentry{app:stiefel_loss}{Proof of Theorem~\ref{thm:similar_carriers}}
\par\noindent\textbf{Setup details.}
\appendixentry{app:setup_details}{Implementation defaults}
\begingroup\small\leftskip=1.5em
\appendixentry{app:runtime_objective}{Training-objective correspondence}
\appendixentry{app:bake}{Bake Training}
\appendixentry{app:residual_weights}{Residual weights}
\endgroup
\par\noindent\textbf{Additional experiments.}
\appendixentry{app:additional_experiments}{Additional experiments overview}
\begingroup\small\leftskip=1.5em
\appendixentry{app:bwt_results}{Backward transfer}
\appendixentry{app:aa_sequence}{Running average accuracy}
\endgroup
\medskip

\section{Proof of Theorem~\ref{thm:vector_orthogonality}}
\label{app:task_vector_bound}

\begin{lemma}[Frobenius Cauchy--Schwarz]
\label{lem:frob_cs}
For real matrices \(\mathbf A,\mathbf B\) of the same shape,
\[
\bigl|\langle\mathbf A,\mathbf B\rangle_F\bigr|
\le \|\mathbf A\|_F\|\mathbf B\|_F.
\]
If \(\mathbf A\) and \(\mathbf B\) can be multiplied, then $\bigl|\operatorname{tr}(\mathbf A\mathbf B)\bigr|\le \|\mathbf A\|_F\|\mathbf B\|_F.$
\end{lemma}
\begin{proof}
Write the entries as vectors. The first bound is the Euclidean Cauchy--Schwarz inequality
\[
\Bigl|\sum_{ab} \mathbf A_{ab}\mathbf B_{ab}\Bigr|
\le
\Bigl(\sum_{ab} \mathbf A_{ab}^2\Bigr)^{1/2}
\Bigl(\sum_{ab} \mathbf B_{ab}^2\Bigr)^{1/2}.
\]
For the second,
\(\operatorname{tr}(\mathbf A\mathbf B)=\langle\mathbf A^{\top},\mathbf B\rangle_F\)
and \(\|\mathbf A^{\top}\|_F=\|\mathbf A\|_F\).
\end{proof}

\begin{proof}[Proof of Theorem~\ref{thm:vector_orthogonality}]

By cyclicity of the trace,
\begin{align*}
\langle\Delta\mathbf W_i,\Delta\mathbf W_j\rangle_F
&=\operatorname{tr}\!\bigl(
\mathbf P_i^{\top}\mathbf W_i^{(d)\top}
\mathbf W_j^{(d)}\mathbf P_j\bigr)\\
&=\operatorname{tr}\!\bigl(
\mathbf W_i^{(d)\top}\mathbf W_j^{(d)}
\mathbf P_j\mathbf P_i^{\top}\bigr).
\end{align*}

Lemma~\ref{lem:frob_cs} then gives
\[
\bigl|\langle\Delta\mathbf W_i,\Delta\mathbf W_j\rangle_F\bigr|
\le
\bigl\|\mathbf W_i^{(d)\top}\mathbf W_j^{(d)}\bigr\|_F
\,\cdot\,
\bigl\|\mathbf P_j\mathbf P_i^{\top}\bigr\|_F.
\]
If \(\mathbf P_i\mathbf P_j^{\top}=\mathbf 0\), then \(\bigl|\langle\Delta\mathbf W_i,\Delta\mathbf W_j\rangle_F\bigr|\) is also zero, and the update inner product vanishes.
\end{proof}

\section{Proof of Theorem~\ref{thm:normalized_interference}}
\label{app:normalized_interference}

Throughout, \(\mathbf E_i=\mathbf P_i-\mathbf P_{i-1}\) is the local carrier increment after task \(i\), \(\widehat{\mathbf E}_i=\mathbf E_i/\|\mathbf E_i\|_F\), and the current historical carrier is
\(\mathbf H_{i-1}=\mathbf P_{i-1}\) with projector
\(\mathbf\Pi_{i-1}=\mathbf H_{i-1}^{\top}\mathbf H_{i-1}\).
The carrier is the rank-\(r\) carrier passed along the chain.

\begin{proof}[Proof of Theorem~\ref{thm:normalized_interference}]
For convenience, denote \(\mathcal L_{g}^{i}\) as \(\epsilon_i\):
\[
\epsilon_i
=\bigl\|\mathbf H_{i-1}\widehat{\mathbf E}_i^{\top}\bigr\|_F
=\frac{\|\mathbf H_{i-1}\mathbf E_i^{\top}\|_F}{\|\mathbf E_i\|_F}.
\]
Thus the argument never assumes \(\|\mathbf E_i\|_F=1\); only
\(||\widehat{\mathbf E}_i||_F=1\) is unit-Frobenius.

Insert \(\mathbf I=\mathbf\Pi_{i-1}+(\mathbf I-\mathbf\Pi_{i-1})\):
\begin{equation}
\widehat{\mathbf E}_j\widehat{\mathbf E}_i^{\top}
=
\widehat{\mathbf E}_j\mathbf\Pi_{i-1}\widehat{\mathbf E}_i^{\top}
+
\widehat{\mathbf E}_j(\mathbf I-\mathbf\Pi_{i-1})\widehat{\mathbf E}_i^{\top}.
\label{eq:app_interference_split}
\end{equation}
For the first term, substitute
\(\mathbf\Pi_{i-1}=\mathbf H_{i-1}^{\top}\mathbf H_{i-1}\):
\[
\widehat{\mathbf E}_j\mathbf\Pi_{i-1}\widehat{\mathbf E}_i^{\top}
=
\bigl(\widehat{\mathbf E}_j\mathbf H_{i-1}^{\top}\bigr)
\bigl(\mathbf H_{i-1}\widehat{\mathbf E}_i^{\top}\bigr).
\]
Lemma~\ref{lem:frob_cs} then gives:
\begin{align*}
\bigl\|
\widehat{\mathbf E}_j\mathbf\Pi_{i-1}\widehat{\mathbf E}_i^{\top}
\bigr\|_F
&\le
\bigl\|\widehat{\mathbf E}_j\mathbf H_{i-1}^{\top}\bigr\|_F
\bigl\|\mathbf H_{i-1}\widehat{\mathbf E}_i^{\top}\bigr\|_F.
\end{align*}
The second factor is \(\epsilon_i\) by definition.
The first factor is at most one:
\[
\bigl\|\widehat{\mathbf E}_j\mathbf H_{i-1}^{\top}\bigr\|_F
\le
\bigl\|\widehat{\mathbf E}_j\bigr\|_F
\bigl\|\mathbf H_{i-1}\bigr\|_F
=1,
\]

The first term is therefore at most \(\epsilon_i\); For the second term the same product bound gives
\begin{align*}
\bigl\|
\widehat{\mathbf E}_j(\mathbf I-\mathbf\Pi_{i-1})\widehat{\mathbf E}_i^{\top}
\bigr\|_F
&\le
\bigl\|\widehat{\mathbf E}_j(\mathbf I-\mathbf\Pi_{i-1})\bigr\|_F
\bigl\|\widehat{\mathbf E}_i^{\top}\bigr\|_F=\bigl\|\widehat{\mathbf E}_j(\mathbf I-\mathbf\Pi_{i-1})\bigr\|_F.
\end{align*}
Since \(\|\widehat{\mathbf E}_i^{\top}\|_F=1\),
so the second term is at most \(\delta_{j\mid i-1}\).
The triangle inequality yields
\(\|\widehat{\mathbf E}_j\widehat{\mathbf E}_i^{\top}\|_F
\le\epsilon_i+\delta_{j\mid i-1}=\mathcal L_{g}^{i}+\delta_{j\mid i-1}\).
\end{proof}

\paragraph{Remarks} The part of the bound cannot remove is
\(\delta_{j\mid i-1}\). It is the piece of \(\widehat{\mathbf E}_j\) outside the common carrier, and it is the obstruction to orthogonality when that carrier is the rank-\(r\) carrier of
every task.
At task \(i\) this carrier is not yet available.
The inherited carrier \(\mathbf H_{i-1}=\mathbf P_{i-1}\) is the best
approximation on hand, so \(\delta_{j\mid i-1}\) is treated as the
incompressible part between the two tasks under that approximation. When that loss \(\mathcal L_g^i\) is close to zero, the interference bound reduces to \(\delta_{j\mid i-1}\): the two normalized increments have been driven as close to orthogonal as this approximation of the carrier allows.

\section{Proof of Theorem~\ref{thm:similar_carriers}}
\label{app:stiefel_loss}
A rotation of the basis inside a fixed subspace changes the numerical entries of a carrier even though the subspace itself has not moved.
If that basis is not brought to the angle at which orthogonality is measured, carriers can already be orthogonal but look coupled, because only the gauge differs.

Forcing the penalty in that mismatched gauge then rotates an already acceptable pair and can itself induce forgetting.
The alignment below removes this rotational factor by rotating the later carrier onto the earlier one, so the diagnostic sees only what is left after that gauge is fixed.

\begin{lemma}[LoRA overlap]
\label{lem:lora_square_overlap}
In the LoRA factors of Section~\ref{sec:lora_correspondence}, \(\mathbf A\in\mathbb R^{r\times m}\) and \(\mathbf B\in\mathbb R^{n\times r}\).
Thin QR of \(\mathbf A^{\top}\) returns a carrier \(\mathbf P\in\mathbb R^{r\times m}\) with \(\mathbf P\mathbf P^{\top}=\mathbf I_r\), so \(\mathbf P\) lies on the Stiefel manifold
\[
\mathrm{St}(r,m)
=\{\mathbf P\in\mathbb R^{r\times m}:\mathbf P\mathbf P^{\top}=\mathbf I_r\}.
\]
Left multiplication by an orthogonal \(r\times r\) matrix stays on \(\mathrm{St}(r,m)\) and does not change the row space: it only rechooses the basis inside the same subspace.
The coefficient is \(\mathbf W^{(d)}=\mathbf B\widehat{\mathbf R}^{\top}\).
For a second carrier \(\mathbf Q\) of the same shape and \(\mathbf E=\mathbf Q-\mathbf P\),
\[
\mathbf P\mathbf E^{\top}\in\mathbb R^{r\times r}.
\]
\end{lemma}

\begin{lemma}[Symmetric splitting]
\label{lem:sym_skew}
Let \(\mathbf M\) be a square matrix of the same shape as \(\mathbf P\mathbf E^{\top}\) in Lemma~\ref{lem:lora_square_overlap}, hence \(r\times r\).
\[
\operatorname{sym}(\mathbf M)=\tfrac12(\mathbf M+\mathbf M^{\top}),
\qquad
\operatorname{skew}(\mathbf M)=\tfrac12(\mathbf M-\mathbf M^{\top}).
\]
Then \(\mathbf M=\operatorname{sym}(\mathbf M)+\operatorname{skew}(\mathbf M)\), and the two parts are orthogonal.
\end{lemma}

Denote the raw squared overlap by
\begin{equation}
\mathcal L
=\|\mathbf P(\mathbf Q-\mathbf P)^\top\|_F^2
=\|\mathbf P\mathbf E^{\top}\|_F^2,
\qquad
\mathbf E=\mathbf Q-\mathbf P.
\label{eq:raw_overlap}
\end{equation}
The orthogonality loss of Eq.~\eqref{eq:lg_def} is this square divided by the squared distance:
\begin{equation}
\mathcal L_g^i
=\frac{\mathcal L}{\|\mathbf E\|_F^2},
\qquad
\mathbf E\neq\mathbf 0.
\label{eq:lg_from_raw}
\end{equation}

We therefore convert the claim to be verified into the following two statements.

(i) We prove that \(\mathcal L\) is the fourth-order bound: $\frac{1}{4r}\|\mathbf E\|_F^4\le\mathcal L\le\frac14\|\mathbf E\|_F^4.$

(ii) Eq.~\eqref{eq:lg_from_raw} then makes \(\mathcal L_g^i\) the second-order bound: $\frac{1}{4r}\|\mathbf E\|_F^2\le\mathcal L_g^i\le\frac14\|\mathbf E\|_F^2.$
\begin{proof}[Proof of Theorem~\ref{thm:similar_carriers}]
Let \(\mathbf P,\mathbf Q\in\mathbb R^{r\times m}\) be row-orthonormal and set \(\mathbf E=\mathbf Q-\mathbf P\), with \(0<\|\mathbf E\|_F\le\eta\).
By Lemma~\ref{lem:lora_square_overlap}, \(\mathbf P\mathbf E^{\top}\) is \(r\times r\).

\paragraph{Procrustes alignment seeks largest \(\operatorname{tr}(\mathbf P\mathbf Q^{\top}\mathbf R)\).}
The rotation in Eq.~\eqref{eq:adapt_svd_alignment} is the first restriction: minimize \(\|\mathbf R^{\top}\mathbf Q-\mathbf P\|_F^2\) over \(\mathbf R\in\mathrm{SO}(r)\).
Expand the square in the Frobenius inner product:
\begin{align*}
\|\mathbf R^{\top}\mathbf Q-\mathbf P\|_F^2
&=\|\mathbf R^{\top}\mathbf Q\|_F^2+\|\mathbf P\|_F^2-2\langle\mathbf R^{\top}\mathbf Q,\mathbf P\rangle_F\\
&=\|\mathbf Q\|_F^2+\|\mathbf P\|_F^2-2\operatorname{tr}(\mathbf P\mathbf Q^{\top}\mathbf R).
\end{align*}
The second uses that \(\mathbf R\) is orthogonal, so \(\|\mathbf R^{\top}\mathbf Q\|_F=\|\mathbf Q\|_F\), together with
\[
\langle\mathbf R^{\top}\mathbf Q,\mathbf P\rangle_F
=\operatorname{tr}\bigl(\mathbf P(\mathbf R^{\top}\mathbf Q)^{\top}\bigr)
=\operatorname{tr}(\mathbf P\mathbf Q^{\top}\mathbf R).
\]
Making the square smaller is therefore the same as making \(\operatorname{tr}(\mathbf P\mathbf Q^{\top}\mathbf R)\) larger.

\paragraph{The largest trace: two situations.} Write the aligned carrier as $\mathbf Q^{\star}=\mathbf R^{\top}\mathbf Q.$
Its cross-Gram is $\mathbf S
=\mathbf P(\mathbf Q^{\star})^{\top}
=\mathbf P\mathbf Q^{\top}\mathbf R.$

Take the SVD of the unaligned cross-Gram, \(\mathbf P\mathbf Q^{\top}=\mathbf U\boldsymbol\Sigma\mathbf V^{\top}\).
Then \(\operatorname{tr}(\mathbf S)=\operatorname{tr}(\boldsymbol\Sigma\mathbf V^{\top}\mathbf R\mathbf U)\), so a rotation maximizes the trace by driving \(\mathbf V^{\top}\mathbf R\mathbf U\) toward \(\mathbf I\).
The factors \(\mathbf V\) and \(\mathbf U\) are that alignment; the two cases below are only whether \(\mathbf V\mathbf U^{\top}\) is already a rotation.

(i) When \(\det(\mathbf V\mathbf U^{\top})=1\), $\mathbf S
=\mathbf U\boldsymbol\Sigma\mathbf V^{\top}\mathbf V\mathbf U^{\top}
=\mathbf U\boldsymbol\Sigma\mathbf U^{\top},$ so $\mathbf S^{\top}=\mathbf S.$

(ii) When \(\det(\mathbf V\mathbf U^{\top})=-1\), the factor \(\mathbf V\mathbf U^{\top}\) is not a rotation. Set \(\mathbf D=\operatorname{diag}(1,\ldots,1,-1)\) and \(\mathbf R=\mathbf V\mathbf D\mathbf U^{\top}\).
Then \(\det(\mathbf R)=1\). With this rotation,
\[
\mathbf S
=\mathbf P(\mathbf Q^{\star})^{\top}
=\mathbf U\boldsymbol\Sigma\mathbf D\mathbf U^{\top}
=\mathbf U\operatorname{diag}(\sigma_1,\ldots,\sigma_{r-1},-\sigma_r)\mathbf U^{\top}.
\]
The new diagonal matrix is still equal to its transpose, so again \(\mathbf S^{\top}=\mathbf S\).
The displacement used below is the difference from this aligned carrier,
\[
\mathbf E=\mathbf Q^{\star}-\mathbf P,
\qquad
0<\|\mathbf E\|_F\le\eta.
\]
Expand the product and use \(\mathbf P\mathbf P^{\top}=\mathbf I_r\):
\[
\mathbf P\mathbf E^{\top}
=\mathbf P(\mathbf Q^{\star}-\mathbf P)^{\top}
=\mathbf P(\mathbf Q^{\star})^{\top}-\mathbf P\mathbf P^{\top}
=\mathbf S-\mathbf I_r.
\]
Both \(\mathbf S\) and \(\mathbf I_r\) are symmetric, so \(\mathbf P\mathbf E^{\top}\) is symmetric and $
\operatorname{skew}(\mathbf P\mathbf E^{\top})
=\tfrac12\bigl(\mathbf P\mathbf E^{\top}-(\mathbf P\mathbf E^{\top})^{\top}\bigr)
=\mathbf 0.$

\paragraph{Induced symmetry part.}
Row orthonormality of \(\mathbf Q=\mathbf P+\mathbf E\) gives $
\mathbf P\mathbf E^{\top}+\mathbf E\mathbf P^{\top}
=-\mathbf E\mathbf E^{\top}.$
The left-hand side is twice the symmetric part, so $
\operatorname{sym}(\mathbf P\mathbf E^{\top})
=-\tfrac12\mathbf E\mathbf E^{\top}.$

\paragraph{Fourth-order bound.}
The skew part is already zero, so \(\mathbf P\mathbf E^{\top}\) equals its symmetric part:
\begin{equation}
\mathbf P\mathbf E^{\top}
=-\tfrac12\mathbf E\mathbf E^{\top},
\qquad
\|\mathbf P\mathbf E^{\top}\|_F^2
=\tfrac14\|\mathbf E\mathbf E^{\top}\|_F^2.
\label{eq:stiefel_decomposition}
\end{equation}
If \(s_1,\ldots,s_q\) are the nonzero singular values of \(\mathbf E\), with \(q\le r\), then \(\|\mathbf E\mathbf E^{\top}\|_F^2=\sum_{k=1}^{q}s_k^4\).
Lemma~\ref{lem:frob_cs} applied to \((s_1^2,\ldots,s_q^2)\) and \((1,\ldots,1)\) gives \(\bigl(\sum_{k=1}^{q}s_k^2\bigr)^2\le q\sum_{k=1}^{q}s_k^4\).
Since \(q\le r\), we have the \textit{left bound}; the cross terms in \(\bigl(\sum_k s_k^2\bigr)^2\) are nonnegative, it holds the \textit{right bound}: 
\[
\frac1r\Bigl(\sum_{k=1}^{q}s_k^2\Bigr)^2
\le \sum_{k=1}^{q}s_k^4
\le \Bigl(\sum_{k=1}^{q}s_k^2\Bigr)^2.
\]
Substitution bounds the raw overlap by $\mathcal L=\frac14\|\mathbf E \mathbf E^\top\|_F^2=\frac14\sum_{k=1}^{q}s_k^4$:
\begin{equation}
\frac{1}{4r}\|\mathbf E\|_F^4
\le
\mathcal L
\le
\frac14\|\mathbf E\|_F^4.
\label{eq:fourth_order_raw}
\end{equation}

\paragraph{Second-order bound.}
Eq.~\eqref{eq:lg_from_raw} divides \(\mathcal L\) by \(\|\mathbf E\|_F^2\).
Dividing Eq.~\eqref{eq:fourth_order_raw} by the same factor gives
\begin{equation}
\frac{1}{4r}\|\mathbf E\|_F^2
\le
\mathcal L_g^i
\le
\frac14\|\mathbf E\|_F^2.
\label{eq:second_order_normalized}
\end{equation}

\paragraph{Read back to the theorem itself.}
Since \(\|\mathbf E\|_F\le\eta\), we have $\mathcal L_g^i
\le\frac14\|\mathbf E\|_F^2
\le\frac14\eta^2.$
\end{proof}

\section{Design Details in LoRA}
\label{app:setup_details}

\subsection{Training-Objective Correspondence}
\label{app:runtime_objective}
With \(L\) adapted layers, training uses the main-text losses and averages
them over layers. The geometric term is
\begin{equation}
\mathcal L_{g}^{i}
=\frac{1}{L}\sum_{\ell=1}^{L}
\frac{\|\mathbf H_{i-1,\ell}\mathbf E_{i,\ell}^{\top}\|_F}
{\|\mathbf E_{i,\ell}\|_F+\gamma},
\label{eq:runtime_geometric_loss}
\end{equation}
where \(\mathbf E_{i,\ell}=\mathbf P_{i,\ell}-\mathbf P_{i-1,\ell}\) and
\(\mathbf H_{i-1,\ell}=\mathbf P_{i-1,\ell}\).
The constant is \(\gamma=10^{-8}\), added so the denominator cannot be zero. The rank term is the same penalty on each LoRA layer,
\begin{equation}
\mathcal L_{\mathrm{rank}}^{i}
=\frac{1}{L}\sum_{\ell=1}^{L}
\frac{1}{r}\sum_{k=1}^{r}
\sqrt{\|\mathbf B_{i,\ell,:,k}\|_2^2+\varepsilon},
\label{eq:runtime_rank_loss}
\end{equation}
where \(\varepsilon=10^{-12}\) keeps the square root finite when a column
vanishes. Because \(\mathbf W^{(d)}=\mathbf B\widehat{\mathbf R}^{\top}\),
these columns are the per-layer form of the \(w_k\) in the main text.

\subsection{Bake training}
\label{app:bake}
During training, the adapter is merged into the weights after each task.
A separate adapter for every task would take more storage as the stream
grows. The next LoRA is then trained above the former merged update, so
earlier tasks remain in the forward pass.

\paragraph{Training forward.}
On task $i+1$, suppressing transpose conventions and dropout,
\[
y=\mathbf W_{\mathrm{0}}x+\operatorname{stopgrad}(\Theta_i)x
+\mathbf B_{i+1}\mathbf A_{i+1}x,
\]
where $\Theta_i$ is the merged update of the tasks seen so far.
This matches Algorithm~\ref{alg:chainlora}.

\subsection{Residual Weights}
\label{app:residual_weights}
Equation~\eqref{eq:merge} writes the merged update as a shared trunk plus
weighted residuals, without explaining \(\omega_i\).
For \(T>1\), considering the distance from the $\mathbf P_0$,
\[
\mathbf D_i=\mathbf P_i-\mathbf P_0,
\qquad
d_i=\|\mathbf D_i\|_2,
\qquad
\bar d=\max\Bigl\{\frac1T\sum_{i=1}^{T}d_i,\,10^{-12}\Bigr\},
\]
and dacays in proportion to how far it sits from $\mathbf P_0$:
\begin{equation}
\omega_i
=\frac{1}{1+d_i/\bar d}
,\qquad
\eta=1.
\label{eq:residual_weight}
\end{equation}
The two branches of Eq.~\eqref{eq:merge} are then
\begin{equation}
\mathbf W_g
=\Bigl(\frac1T\sum_{i=1}^{T}\mathbf W_i^{(d)}\Bigr)\mathbf P_0,
\qquad
\mathbf W_{\mathrm{sp}}
=\sum_{i=1}^{T}\omega_i\mathbf W_i^{(d)}\mathbf D_i,
\label{eq:residual_branches}
\end{equation}
and the absorbed update scales only the trunk by the LoRA factor,
\[
\Delta\mathbf W=\frac{\alpha}{r}\mathbf W_g+\mathbf W_{\mathrm{sp}}.
\]

\section{Additional Experiments}
\label{app:additional_experiments}

\subsection{Backward Transfer}
\label{app:bwt_results}

\begin{table}[h]
\centering\small
\caption{Stored final BWT for the model-size grid in
Table~\ref{tab:model_size}. The lower value indicates better performance in hindering forgetting.}
\label{tab:bwt_model_size}
\resizebox{\textwidth}{!}{
\begin{tabular}{l|cccc|cccc|ccc}
\toprule
& \multicolumn{4}{c|}{\textbf{Standard CL}}
& \multicolumn{4}{c|}{\textbf{Large number of tasks}}
& \multicolumn{3}{c}{\textbf{SuperNI}} \\
\textbf{Model}
& O1 & O2 & O3 & avg
& O4 & O5 & O6 & avg
& O1 & O2 & avg \\
\midrule
Llama-3.2-3B
& $-0.1$ & $-7.0$ & $-5.7$ & $-4.3$
& $-0.3$ & $-2.9$ & $-4.9$ & $-2.7$
& $-16.7$ & $-24.6$ & $-20.6$ \\
Llama-2-7B-chat
& $-0.6$ & $-1.6$ & $-1.9$ & $-1.4$
& $+0.4$ & $-4.4$ & $-5.7$ & $-3.2$
& $-10.8$ & $-7.1$ & $-8.9$ \\
Llama-2-13B-chat
& $-5.5$ & $-7.1$ & $-6.7$ & $-6.4$
& $-10.6$ & $-22.0$ & $-13.7$ & $-15.4$
& $-23.7$ & $-30.8$ & $-27.3$ \\
\bottomrule
\end{tabular}
}
\end{table}
\vspace{-10pt}
\subsection{Running average accuracy}
\label{app:aa_sequence}
\begin{figure}[htbp]
  \centering
  \includegraphics[width=\linewidth]{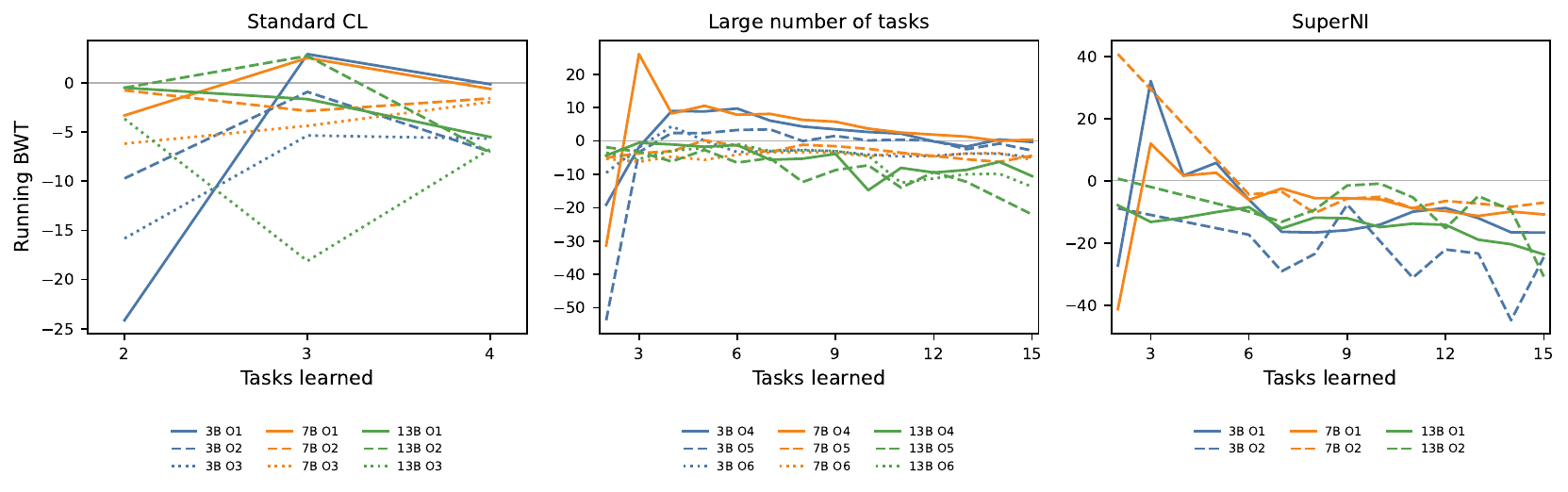}
  \caption{Running BWT after each completed task, for the runs in
  Table~\ref{tab:bwt_model_size}.}
  \label{fig:bwt_running}
\end{figure}
\begin{figure}[htbp]
  \centering
  \includegraphics[width=\linewidth]{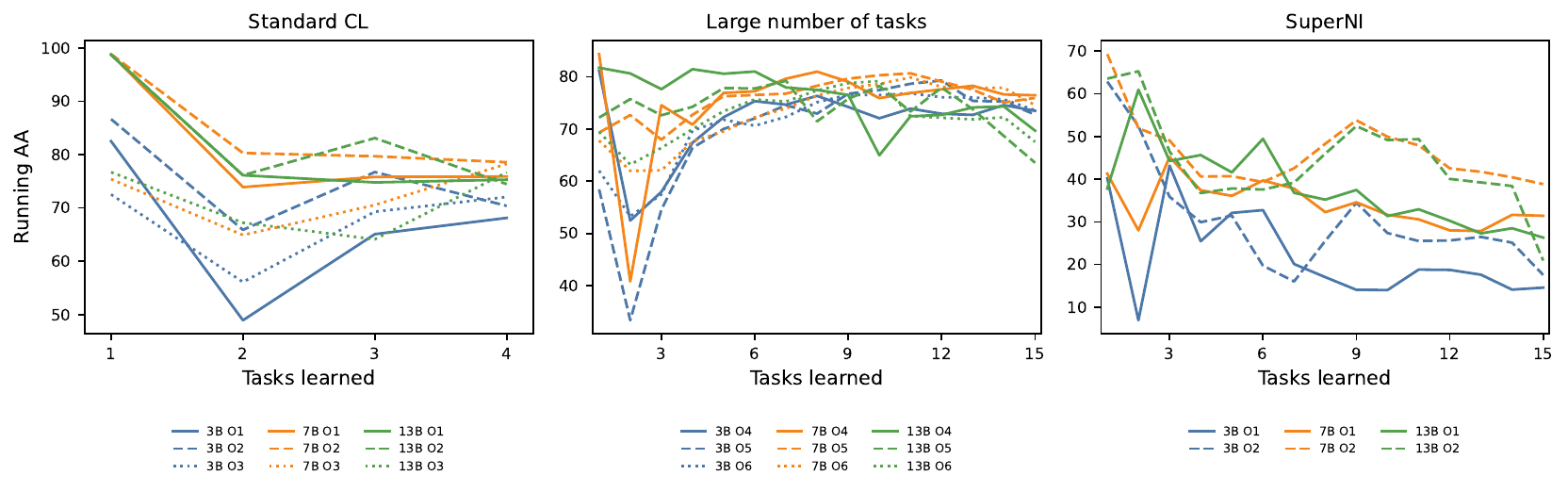}
  \caption{Running AA after each completed task, for the runs in Table~\ref{tab:bwt_model_size}. At step \(t\), \(\mathrm{AA}_t\) averages one score for each task seen so far.}
  \label{fig:aa_running}
\end{figure}

\end{document}